\documentclass{article} %
\usepackage{template_conference}

\usepackage{microtype}
\usepackage{hyperref}
\usepackage{url}
\usepackage{booktabs}

\title{Asymptotics of Language Model Alignment }

\author{%
  Joy Qiping Yang \\ University of Sydney\\
                    Sydney, Australia\\     \texttt{qyan6238@uni.sydney.edu.au}\\
 \And
  Salman Salamatian\\
 Massachusetts Institute of Technology\\
                    Cambridge, MA, USA\\    \texttt{salmansa@mit.edu}\\
\And
 Ziteng Sun, Ananda Theertha Suresh, Ahmad Beirami\\
  Google Research\\
                    New York, NY, USA\\    
                    \texttt{\{zitengsun, theertha, beirami\}@google.com}
}

\usepackage{bm}
\usepackage{comment}
\usepackage{xcolor}
\usepackage{amsfonts}
\usepackage[utf8]{inputenc} 
\usepackage[T1]{fontenc}
\usepackage{url}
\usepackage{ifthen}
\usepackage{cite}
\usepackage{graphicx}
\usepackage[cmex10]{amsmath}
\usepackage{hyperref}
\usepackage{url}

\newcommand{\ssedit}[1]{{\noindent \textcolor{black}{#1}}}

\definecolor{ruddy}{rgb}{1.0, 0.0, 0.16}
\definecolor{gblue}{RGB}{29, 144, 255}
\definecolor{royalblue}{rgb}{0.25, 0.41, 0.88}
\definecolor{edits}{rgb}{1.0,0.0,0.0}

\hypersetup{
  colorlinks,
  citecolor=royalblue,
  linkcolor=royalblue,
  urlcolor=royalblue}

\usepackage{bm}
\usepackage{hyperref}
\usepackage[capitalise]{cleveref}
\usepackage{thm-restate}
\usepackage{amsthm}
\usepackage{amssymb}
\newtheorem{definition}{Definition}
\newtheorem{theorem}{Theorem}
\newtheorem{lemma}{Lemma}
\newtheorem{corollary}{Corollary}
\newtheorem{assumption}{Assumption}
\newtheorem{example}{Example}

\newcommand{\bx}{\bm{x}}
\newcommand{\by}{\bm{y}}
\newcommand{\bz}{\bm{z}}
\newcommand{\bmu}{\bm{p}}
\newcommand{\bpi}{\bm{\pi}}
\newcommand{\bomega}{\bm{\omega}}
\newcommand{\bphi}{\bm{\phi}}
\newcommand{\bpsi}{\bm{\psi}}
\newcommand{\cx}{\mathcal{X}}
\newcommand{\cy}{\mathcal{Y}}
\newcommand{\bp}{\bm{p}}
\newcommand{\bq}{\bm{q}}
\newcommand{\br}{\bm{r}}

\newcommand{\ahmad}[1]{{\noindent \textit{\small\textcolor{blue}{ahmad: #1}}}}

\newcommand{\E}{\mathbb{E}}
\newcommand{\Type}{t}
\newcommand{\KL}{D_{\operatorname{KL}}}
\renewcommand{\S}{\mc{S}^K_\zeta}
\newcommand{\forfuture}[1]{
}

\newcommand{\mc}{\mathcal}

\colmfinalcopy %
\begin{document}

\maketitle

\begin{abstract}
Let $\bp$ denote a reference generative language model. Let $\br$ denote a reward model that returns a scalar to capture the degree at which a draw from $\bp$ is preferred. 
The goal of {\em language model alignment} is to  alter $\bp$ to a new distribution  $\bphi$ that results in a higher expected reward while keeping $\bphi$ close to $\bp.$ A popular alignment method is the {\em KL-constrained reinforcement learning (RL)}, which chooses a distribution $\bphi_\Delta$ that maximizes $E_{\bphi_{\Delta}} \br(\by)$ subject to a relative entropy constraint $\KL(\bphi_\Delta \| \bp) \leq \Delta.$ Another simple alignment method is {\em best-of-$N$}, where $N$ samples are drawn from $\bp$ and one with highest reward is selected. %
In this paper, we offer a closed-form characterization of the optimal KL-constrained RL solution. We then demonstrate that any alignment method that achieves a comparable trade-off between KL divergence and expected reward must approximate the optimal KL-constrained RL solution in terms of relative entropy.
To further analyze the properties of alignment methods, we introduce two simplifying assumptions: we let the language model be memoryless, and the reward model be linear.
Although these assumptions may not reflect complex real-world scenarios, they enable a precise characterization of the asymptotic (in the sequence length) behavior of both the best-of-$N$ alignment, and the KL-constrained RL method, in terms of information-theoretic quantities.
We prove that the reward of the optimal KL-constrained RL solution satisfies a large deviation principle, and we fully characterize its rate function. We also show that the rate of growth of the scaled cumulants of the reward is characterized by a proper R\'enyi cross entropy. Finally, we show that best-of-$N$ is asymptotically equivalent to KL-constrained RL solution by proving that their expected rewards are asymptotically equal, and concluding that the two distributions must be close in KL divergence.
\end{abstract}

\section{Introduction}

Generative models, such as large language models~\citep{brown2020language, team2023gemini,touvron2023llama} and diffusion models~\citep{sohl2015deep, ho2020denoising, song2020score}, are increasingly popular general-purpose problem solvers.
These models are generally trained on large corpora of unlabeled data; with additional training on domain-specific data through supervised finetuning (SFT). We refer to the resulting model as the {\em reference model}. The reference model, while capable of generating realistic outputs, may still generate undesirable outcomes, e.g., unsafe language or images. Hence, it is desirable to further {\em align} the output of the reference model with a reward that captures human preferences~\citep{christiano2017deep,stiennon2020learning,ouyang2022training, bai2022training}.

The literature has seen a proliferation of algorithmic techniques for alignment~\citep{christiano2017deep,stiennon2020learning,zhao2022calibrating, ouyang2022training, bai2022training,rafailov2023direct, mudgal2023controlled,yang-klein-2021-fudge}.
Two of the popular methods are: (a) {\em KL-constrained reinforcement learning (RL), $\bphi_\Delta$,}~\citep{christiano2017deep,ouyang2022training}, which maximizes expected reward subject to a KL divergence constraint, $\Delta,$ between the aligned model and the reference model; and (b) {\em Best-of-$N$, $\bpi_N$,}~\citep{nakano2021webgpt,touvron2023llama}, 
where $N$ i.i.d. samples are drawn from the reference model and one with the highest reward is returned. 

Despite the popularity of both of these alignment methods, no relationship between them is known. 
In particular, best-of-$N$ has shown to be surprisingly competitive with, or even outperform more involved KL-constrained RL algorithms that aim at achieving the best reward-KL tradeoffs~\citep{mudgal2023controlled, gao2023scaling,eisenstein2023helping}.
In this work, we provide theoretical reasoning for this empirical observation by showing that the two methods are in fact asymptotically equivalent in some specific settings. We hope our work initiates further study into the relationship between different alignment methods. The summary of our contributions may be found below:

\begin{itemize}
    \item We characterize the unique optimal solution of the KL-constrained reinforcement learning problem, denoted by $\bphi_\Delta,$ for a given KL constraint of $\Delta$. We show that the solution is a {\em mismatched tilted distribution} that had appeared and characterized by~\citet{salamatian2019mismatched} in the context of mismatched guesswork.

    \item We show that if an alignment method $\bomega$ satisfies a KL constraint $\Delta$, and also achieves {\em almost} the same expected reward as $\bphi_\Delta$ (the optimal KL-constrained RL solution), then $\KL(\bomega \| \bphi_\Delta)$ must be vanishingly small, where $\KL(\cdot \| \cdot)$ denotes the KL divergence and is formally defined in~\eqref{eq:define-KL-H}.

    \item We show that the cumulants of the reward are related to the R\'enyi cross entropy of the alignment distribution to the source distribution.
    Under assumptions on the data generating process, we also show that the optimal KL-constrained RL solution, $\bphi_\Delta,$ satisfies a large deviation principle, where we characterize its rate function using information-theoretic quantities. 

    \item Under assumptions on the data-generating process, we show that the outcomes of the best-of-$N$ alignment method, $\bpi_N,$ are similar to those of the optimal KL-constrained RL, $\bphi_\Delta,$ for $N = \exp(\Delta)$. In particular, their expected rewards are also close, which implies that KL-constrained RL method is close to the best-of-$N$ method in KL divergence, i.e., $\KL(\bpi_N \| \bphi_\Delta)$ is vanishingly small.
\end{itemize}

\section{Problem Setup}
Let $\bx \in \bm{\mc{X}}$ denote an input prompt, e.g., {\em Who is Barack Obama?} 
Let $\bm{\mc{Y}}$ be the set of possible outputs of a generative language model. A language model $\bp_{\bm{y}|\bm{x}}$ assigns probability distributions over $\bm{\mc{Y}}$ given a context (input prompt) $\bx \in \bm{\mc{X}}$ i.e., $\bp_{\bm{y}|\bm{x}} : \bm{\mc{X}} \times \bm{\mc{Y}} \to \mathbb{R}^{\geq0}$ such that $\forall \bx$
\[
\sum_{\bm{y} \in \bm{\mc{Y}}} \bp_{\bm{y}|\bm{x}}(\by | \bx) = 1.
\]
Throughout this paper, we use $\bp$ to denote the reference language model which we wish to further align.

A reward model $\br$ assigns scalar scores for every prompt and output i.e., $\br : \bm{\mc{X}} \times \bm{\mc{Y}} \to \mathbb{R}$. Given such a reward model, one can also define an alignment distribution as follows.

\begin{definition}[alignment distribution]
\label{definition:alignment_distribution}
  For a given reward function $\br$, we denote $\bq_{\by|\bx}$ as the corresponding alignment distribution defined as
\[
\bq_{\by|\bx} (\by | \bx) = \exp \left(  \br(\bx, \by )\right)/ R_{\bx},
\]
where $R_{\bx}$ is the partition function, i.e., $R_{\bx} := \sum_{\by \in \bm{\mc{Y}}}\exp \left(  \br(\bx, \by )\right)$. Without loss of generality, and for convenience in the rest of this paper we assume that $R_{\bx} = 1,$ i.e., reward is the negative log likelihood of a generative distribution.
\end{definition}
For two distributions $\bp$ and $\bq$, let $H(\bp \| \bq)$ denote the cross-entropy of $\bp$ to $\bq$; and $\KL(\bp \| \bq)$ denote the Kullback–Leibler divergence. For discrete distributions $\bp$ and $\bq$ over the same alphabet $\bm{\mc{Y}}$, we have that\footnote{All logarithms in this paper are in base $e.$} 
\begin{equation}
    H(\bp \| \bq) = \sum_{\by \in \bm{\mc{Y}}} p(\by) \log \frac{1}{q(\by)}; \quad\quad \KL(\bp \| \bq) = \sum_{{\by \in \bm{\mc{Y}}}} p(\by) \log \frac{p(\by)}{q(\by)}.
\label{eq:define-KL-H}
\end{equation}

With these definitions, the KL-constrained RL problem is defined as follows:
\begin{definition}[KL-constrained RL]
\label{definition:RL_alignment}
    We say $\bphi_\Delta$ is an optimal aligned model of $\bm{p}$ with KL divergence $\Delta,$ if
    \begin{equation}
        \bphi_\Delta (\cdot | \bx) \in \arg\min_{\bphi \in \mc{D}_{\Delta}(\bp)} H(\bphi (\cdot | \bx) \| \bq(\cdot | \bx)),
    \end{equation}
where $\mc{D}_\Delta(\bp) := \{ \bphi :  \KL(\bphi(\cdot | \bx) \| \bp(\cdot | \bx)) \leq \Delta \}$.
\end{definition}

Under the KL-constrained RL method, outputs are then generated using $\bm{\phi}_\Delta$. 
We remark that 
\begin{align*}
 H(\bphi (\cdot | \bx) \| 
 \bq (\cdot | \bx)) 
& =  - \E_{\by \sim \bphi (\cdot | \bx)}  \log \bq_{\by|\bx}(\by|\bx) = - \E_{\by \sim \bphi
 (\cdot | \bx)} \br(\bx, \by),
\end{align*}
and hence minimizing  the cross entropy $H(\bphi(\cdot | \bx) \| \bq (\cdot | \bx))$  is equivalent to maximizing expected reward $\E_{\by \sim \bphi
 (\cdot | \bx)} \br(\bx, \by)$, which is the common terminology in the language model alignment literature~\citep{gao2023scaling,eisenstein2023helping,mudgal2023controlled}. 

Next, we describe the best-of-$N$ alignment method. 
\begin{definition}[best-of-$N$]
 Let $\by_1, \ldots, \by_N$ be $N$ i.i.d. draws from $\bp_{\by | \bx}(\cdot | \bx)$. The best-of-$N$ alignment policy returns\footnote{We define $[N] := \{1, \ldots, N\}$.}
\begin{equation}
    \by = \by_{k^*}  \quad\quad \text{where}\quad \quad k^* := \arg\max_{k \in [N]}  \bq_{\by | \bx}(\by_k |\bx).
\end{equation}
Further, let $\bpi_N$ denote the probability distribution of the best-of-$N$ strategy. It is clear that $\bpi_1 = \bp$. 
\end{definition}
The best-of-$N$ alignment method generates $N$ candidate outputs directly from the reference model and selects the one with highest reward. Further, its PMF is characterized by~\citet[Lemma 1]{beirami2024theoretical}.

In empirical studies~\citep{gao2023scaling, mudgal2023controlled, eisenstein2023helping, rafailov2023direct}, in order to compare a family of decoding policy $\{\bpsi_\beta\}$ parameterized by $\beta$, it is common to report tradeoff curves of {\em expected reward} $E_{\bx} E_{\bpsi_\beta} \{\br(\bx,\by)\}$ vs {\em KL divergence} $E_{\bx} \KL(\bpsi_\beta(\cdot | \bx) \| \bp(\cdot | \bx)),$ which we refer to as {\em reward-KL tradeoff curves} (e.g.,~\citep[Fig. 3]{mudgal2023controlled}). A decoding policy is desired when its reward-KL tradeoff curve dominates that of any other decoding policy. 
The optimal aligned model (Definition~\ref{definition:RL_alignment}) must dominate any other alignment technique by definition. However, in practice best-of-$N$ achieves remarkably good reward-KL tradeoffs and often dominates those of sophisticated methods solving the RL problem in practical scenarios~\citep{mudgal2023controlled, gao2023scaling}. One of our major objectives in this paper is to develop a theoretical understanding for this phenomenon.

\section{Optimal KL-Constrained RL Solution}

Our first goal is to understand the precise theoretical behavior of the reward under the optimal KL-constrained RL. Our second goal is to study the relationship between the optimal KL-constrained RL and best-of-$N$ methods. We want to answer the questions: {\em does best-of-$N$ method produce good or even optimal solutions to the KL-constrained RL problem? Is there any theoretical support for the good empirical performance of best-of-$N$?} 

First, we characterize the solution to the KL-constrained RL problem.
\begin{lemma}[optimal aligned model]
    \label{lemma:solution_to_KL_RL}
   The solution of the KL-constrained RL problem in Definition~\ref{definition:RL_alignment} is unique and is given by
    \begin{equation}        \bphi_\Delta(\cdot | \bx) = T(\bq(\cdot | \bx), \bp(\cdot | \bx), \alpha(\Delta)),
    \label{eq:mismatched-tilt}
    \end{equation}
    where $T(\bq(\cdot | \bx), \bp(\cdot | \bx), \alpha)$ is the mismatched tilt~\citep[Definition 1]{salamatian2019mismatched}:
 \begin{equation}
     T(\bq(\cdot | \bx), \bp(\cdot | \bx), \alpha)(\by) := \frac{\bp(\by | \bx)\bq^{\alpha}(\by | \bx)  }{Z_{\bx, \alpha}},
 \end{equation}   
    where 
    \begin{equation}
        Z_{\bx, \alpha} :=\sum_{\bz \in \bm{\mathcal{Y}}} \bp(\bz | \bx)\bq^{\alpha}(\bz | \bx),
    \end{equation}
    and
    \begin{equation}
        \alpha(\Delta) := \arg_{\alpha \in \mathbb{R}^{\geq 0}} \left\{\KL(T(\bq(\cdot | \bx), \bp(\cdot | \bx), \alpha) \| \bp) = \Delta \right\}.
        \label{eq:alpha-Delta}
    \end{equation}
\end{lemma}
\begin{proof}[Proof sketch]
This follows from applying the Karush–Kuhn–Tucker (KKT) conditions on the convex problem in Definition~\ref{definition:RL_alignment}.
\end{proof}

A variant of this result has already appeared in~\citep{korbak2022reinforcement, korbak2022rl} and we provide a proof in the appendix for completeness. 
The result above suggests that the aligned models (parameterized by $\Delta$) form an exponential family of distributions,\footnote{This specific family of distribution also arises in characterization of other elements in information theory, e.g. in mismatched guesswork and mismatched one-to-one source coding~\citep[Definition 1]{salamatian2019mismatched}, and is referred there as the mismatched tilted family of distributions.} and we define this family of solutions to the KL-constrained RL for different values of $\Delta$ as the aligned family.
\begin{definition}[aligned family]
    Let the aligned family of $\bp$ towards $\bq$ be denoted by $\mc{T}(\bp, \bq)$, and be given by
    \begin{equation}
        \mc{T}(\bp, \bq) = \{\bphi_\Delta\}_{\Delta \geq 0}.
    \end{equation}
\end{definition}

Equipped with the optimal aligned model for a given KL divergence constraint, we ask how close a model would be to optimal, if it offered almost the same expected reward as the optimal model.
\begin{lemma}[closeness of models given closeness of expected rewards]
    Let $\bpsi_\Delta$ be such that 
    \begin{equation}
     \KL(\bpsi_\Delta \| \bp) \leq \Delta; \quad\quad   H(\bpsi_\Delta \|\bq) \leq H(\bphi_\Delta \|\bq) + \epsilon,
    \label{eq:lemma-closeness-assumptions}
    \end{equation}
where $\bphi_\Delta$ is the optimal aligned model at KL divergence $\Delta$.
Then, 
\begin{equation}
    \KL(\bpsi_\Delta \| \bphi_\Delta) \leq \alpha(\Delta) \epsilon,
\end{equation}
\label{lem:almost-optimal-aligned}
where $\alpha(\Delta)$ is defined in~\eqref{eq:alpha-Delta}.
\end{lemma}
\begin{proof} Note that
\begin{align}
    \KL(\bpsi_\Delta(\cdot | \bx) \| \bphi_\Delta(\cdot| \bx)) & =   \KL(\bpsi_\Delta(\cdot | \bx) \| \bphi_\Delta(\cdot| \bx)) - \KL(\bphi_\Delta(\cdot| \bx) \| \bphi_\Delta(\cdot| \bx)) \nonumber\\ &= \mathbb{E}_{\by \sim \bpsi_\Delta(\cdot | \bx)} \log \frac{\bpsi_\Delta(\by| \bx)}{\bp(\by| \bx)} + \alpha(\Delta) \mathbb{E}_{\by \sim \bpsi_\Delta(\cdot| \bx) } \log \frac{1}{\bq(\by| \bx)} + \log Z_{\bx, \alpha(\Delta)}\nonumber\\
    & - \mathbb{E}_{\by \sim \bphi_\Delta(\cdot | \bx)} \log \frac{\bphi_\Delta(\by| \bx)}{\bp(\by| \bx)} - \alpha(\Delta) \mathbb{E}_{\by \sim \bphi_\Delta(\cdot| \bx) } \log \frac{1}{\bq(\by| \bx)} - \log Z_{\bx, \alpha(\Delta)}\label{eq:step1} \\
    & = \KL(\bpsi_\Delta(\cdot | \bx) \| \bp(\cdot \| \bx) ) - \KL(\bphi_\Delta(\cdot| \bx)\| \bp(\cdot \| \bx) )\nonumber\\
    & + \alpha(\Delta) \left( H(\bpsi_\Delta \|\bq) - H(\bphi_\Delta \|\bq) \right) \nonumber\\
    & \leq \alpha(\Delta) \epsilon, \label{eq:step2}
\end{align}
where~\eqref{eq:step1} follows from the fact that $\bphi_\Delta$ could be expressed by~\eqref{eq:mismatched-tilt} in Lemma~\ref{lemma:solution_to_KL_RL}, and~\eqref{eq:step2}
follows from the assumptions of the lemma in~\eqref{eq:lemma-closeness-assumptions}, which completes the proof.
\end{proof}
Lemma~\ref{lem:almost-optimal-aligned} states that any model that satisfies KL constraint $\Delta,$ and whose expected reward is within $\epsilon$ of that of the optimal aligned model must be close to the optimal aligned model in KL divergence. We will use this result to establish such closeness between best-of-$N$ and the optimal KL-constrained model in the sequel.

Before we proceed, we would like to also observe a connection between the cumulants of the reward and R\'enyi cross entropy.
\begin{definition}[R\'enyi cross entropy]
For all $t> 0,$ let R\'enyi cross entropy of order $t$ be defined as\footnote{The definition is extended via continuous extension at $t = 1.$}~\citep[Eq. (38)]{li2023tilted}
\begin{equation}
    H_t (\bp \| \bq) := \frac{1}{1-t} \log \sum_{\by \in \bm{\mathcal{Y}}} \bp(\by) \bq(\by)^{t -1}.
\end{equation}    
\end{definition}
R\'enyi cross entropy of order $1$ recovers the cross entropy, i.e., $H_1 (\bp \| \bq) = H (\bp \| \bq).$ 

\begin{lemma}[cumulant generating function of reward]
    Let $ E_{\Delta, \rho}(\bp, \bq)$ be the $\rho$-th cumulant of the reward under the optimal $\Delta$-aligned model, and for all $\rho \geq 0, $be defined as\footnote{For $\rho = 0,$ the definition is extended via continuous extension.}
\begin{equation}
    E_{\Delta, \rho}(\bp, \bq) :=  
    \frac{1}{\rho} \log \mathbb{E}_{\by \in \bphi_{\Delta} (\cdot|\bx)} \bq_{\by|\bx} (\by |\bx)^\rho = \frac{1}{\rho} \log \mathbb{E}_{\by \in \bphi_{\Delta} (\cdot|\bx)} \exp ( \rho \, \br (\bx, \by )),
\end{equation}
where the latter equality is due to the fact that we assumed $R_{\bx} = 1.$
Then,
\begin{equation}
    E_{\Delta, \rho}(\bp, \bq) = - H_{1+\rho}(\bphi_\Delta \| \bq).
\end{equation}
\label{lem:cumulant-general}
\end{lemma}
\begin{proof}
    By definition,
\begin{equation}
    H_{1+\rho}(\bphi_\Delta \| \bq) = - \frac{1}{\rho} \log \mathbb{E}_{\by \in \bphi_{\Delta} (\cdot|\bx)} \bq_{\by|\bx} (\by |\bx)^\rho,
\end{equation}
which completes the proof.
\end{proof}

In particular, 
\begin{equation}
    E_{\Delta, 0}(\bp, \bq) = \mathbb{E}_{\by \in \bphi_{\Delta} (\cdot|\bx)} \log  \bq_{\by |\bx} (\by | \bx) = \mathbb{E}_{\by \in \bphi_{\Delta} (\cdot|\bx)} \br (\bx, \by) =  - H(\bphi_\Delta \| \bq).
\end{equation}

\section{Main  Results}
Equipped with the preliminaries, in this section we derive our main results. To this end, we make two simplifying assumptions.

\begin{assumption}[memoryless reference model]
    We assume that $\bp_{\by|\bx}$ is a memoryless source such that the outcome is a sequence of length $m$ from the $m$-product of the categorical distribution defined by stochastic vector $p \in \S$, where      $\S$ denotes the interior of the simplex over alphabet size $K,$ such that \begin{equation}
       \S := \left\{p : \sum_{k \in [K]} p_k = 1, \quad p_k > \zeta, \quad \forall k \neq k'~~ p_k \neq p_{k'}\right\}. 
    \end{equation}
\label{assump:memoryless}
\end{assumption}

\begin{assumption}[linear reward]
 We assume that the reward is the negative loglikelihood of an alignment distribution $\bq_{\by|\bx}$ that is memoryless and a product of categorical distribution $q \in \S$, i.e., $\bq_{\by|\bx}$ satisfies Assumption~\ref{assump:memoryless}. This immediately implies that $\br$ is bounded from above and is linear in the outcome.
 \label{assump:linear-reward}
\end{assumption}

 We adopt Assumptions~\ref{assump:memoryless}-\ref{assump:linear-reward} for brevity in this paper, and comment on extension to more general classes of language models and reward functions. Also note that the regularity condition on $\bq_{\by|\bx}$ is equivalently a regularity condition on the reward function (i.e. that it is bounded, additive, and assigns distinct reward to all of possible types).
\begin{definition}[type]
Let $y^m$ be a sequence of length $m$ supported on alphabet of size $K.$ Then the type of sequence $y^m$ is denoted by $\Type(y^m)$ defined as
\begin{equation}
    \Type(y^m) = \left(\frac{1}{m}\sum_{i \in [m]} \mathbf{1}\{y_i = 1\}, \ldots, \frac{1}{m}\sum_{i \in [m]} \mathbf{1}\{y_i = K\}\right),
\end{equation}
where $\mathbf{1}\{\cdot\}$ is the indicator function.
\label{def:type}
\end{definition}
\ssedit{The type of a sequence denotes its empirical distribution and is a sufficient statistic of the sequence under the memoryless assumption for the reference model.}

In this setup, our goal is to show that both best-of-$N$ and optimal KL-constrained RL method return sequences whose type (empirical distribution) is close to the intersection of the reward contour and the KL divergence contour. To this end, we first start with the KL-constrained RL solution.

\subsection{KL-Constrained RL Problem}
Under the memoryless setup of Assumption~\ref{assump:memoryless}-\ref{assump:linear-reward}, the aligned family and optimal KL-constrained RL solution can be further simplified as follows.
\begin{lemma}[optimal aligned model is in the same family as the reference model and reward model]
    \label{lemma:solution_to_KL_RL_iid}
    If Assumptions~\ref{assump:memoryless}-\ref{assump:linear-reward} hold, then for each $\bx$, $\bphi_\Delta( y^m | \bx)$ is also a memoryless distribution that satisfies Assumption~\ref{assump:memoryless}, i.e., 
 there exists a distribution $\phi_\delta$ such that $\Delta = m\delta$ and
    \[
\bphi_\Delta( y^m | \bx) = \prod^m_{i=1} \phi_\delta(y_i).
    \]
Furthemore, $ \phi_\delta$ is given by
    \begin{equation}
        \phi_\delta = T(q, p, \alpha(\delta)),
    \end{equation}
    where $T(q, p, \alpha)$ is the mismatched tilt~\citep[Definition 1]{salamatian2019mismatched}:
 \begin{equation}
     T(q, p, \alpha)(y) = \frac{p_i q_i^{\alpha}}{\sum_{j \in [K]} p_j q_j^{\alpha}}
 \end{equation}   
    and where
    \begin{equation}
        \alpha(\delta) := \arg_{\alpha \in \mathbb{R}+} \left\{\KL(T(q, p, \alpha) \| p) = \delta \right\}.
    \end{equation}
\end{lemma}
\begin{proof}
The proof follows from Lemma~\ref{lemma:solution_to_KL_RL}
and by identifying that:
\begin{align*}
H(\bphi_\Delta \| \bq) & = H(\phi^m_\delta \| q^m) = m H (\phi_\delta \| q), \\
\KL(\bphi_\Delta \| \bq) & = \KL(\phi^m_\Delta \| q^m) = m \KL (\phi_\delta \| q) .
\end{align*}
\end{proof}

\begin{figure}[t]
    \centering
    \includegraphics[width=0.8\linewidth]{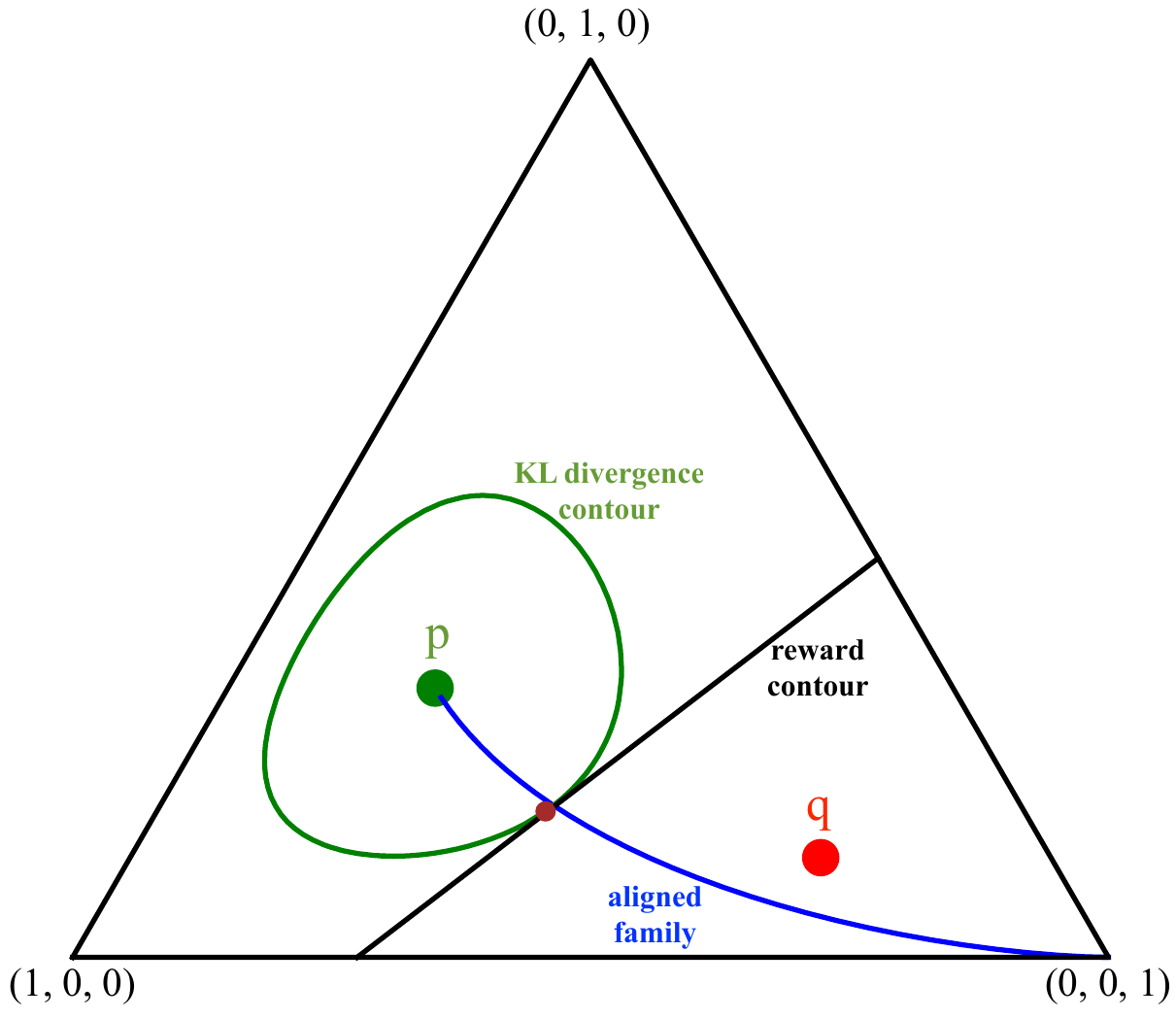}
    \caption{We consider a ternary memoryless source, where the green point is $p = (\frac{1}{5}, \frac{3}{10}, \frac{1}{2})$, the red point is $q = (\frac{2}{3}, \frac{1}{9}, \frac{2}{9})$. The green curve characterizes all $v$ s.t. $\KL(v \| p) = \Delta = 0.11$, which is the KL contour. The black line characterizes all $w$ s.t. $H(w\|q) = H(T(q, p, \alpha(0.11) \| q),$ which is a linear family. The blue curve is $\mathcal{T}_{p,q},$ i.e., the aligned family. The unique intersection of the green KL ball and the black constant reward front is a point on the blue curve which is $\phi_\Delta.$
    The brown point is the expected value of the type of the best-of-$N$ for $N = \exp(m \delta) \approx 3,$ where $m = 10,$ which is remarkably close to $\phi_\Delta.$ }
    \label{fig:simplex}
\vspace{-.1in}
\end{figure}

\begin{example}[ternary memoryless source]
Let us introduce a recurring example of such source on a ternary alphabet $\mc{Y} = \{1, 2, 3\},$ with $K =3.$ Let $p = (\frac{1}{5}, \frac{3}{10}, \frac{1}{2})$ and $q = (\frac{2}{3}, \frac{1}{9}, \frac{2}{9})$. In Figure~\ref{fig:simplex}, the two distributions are depicted on the ternary probability simplex. A KL divergence contour $\{\phi: \KL(\phi \| p) =  \delta\}$ is depicted via the green curve. A reward contour $\{\phi: H(\phi \|p) = c\}$ is depicted via the black line. The aligned family $\{\phi_\delta\}_{\delta \geq 0}$ is depicted via the blue line. %
\end{example}

Equipped with this result, we now show that the sequence of optimal aligned models as a function the sequence length, $m,$ satisfies a large deviation principle (LDP).
\begin{definition}[large deviation principle (LDP)]
A sequence of random variables $\{Y_n:n\in\mathbb{N}\}$ taking values in $\mathbb{R}$ satisfies the Large Deviation Principle with rate function $J:\mathbb{R}\to[0,\infty]$ if $J$ is lower semi-continuous and has compact level-sets, and for all Borel sets $B$
\begin{align*}
-\inf_{t\in B^\circ} J(t)
        & \leq \liminf_{n\to\infty} \frac 1n \log \mathbb{P}\{Y_n\in B\} \leq \limsup_{n\to\infty} \frac 1n \log \mathbb{P}\{Y_n\in B\}
        \leq -\inf_{t\in \bar{B}} J(t),
\end{align*}
where $B^\circ$ is the interior of $B$ and $\bar{B}$ is its closure.
\end{definition}
\ssedit{The rate function of the Large Deviation Principle (LDP) thoroughly characterizes the probability of rare events in which $Y_n$ significantly deviates from its mean. Specifically, an LDP offers a complete description of the cumulant generating function of $Y_n$ by employing Varadhan's Lemma~\citep[Theorem 4.3.1]{Dembo}.}
We are now ready to state our next result which is an LDP for the reward function for sequences generated from $\bphi_\Delta$, and under  Assumptions~\ref{assump:memoryless}-\ref{assump:linear-reward}.

\begin{theorem}[LDP for aligned reward]
Let $Y^m \sim \phi_\delta^m.$
    The sequence $\{ -\frac{1}{m} \log q^m(Y^m)\}_{m \in \mathbb{N}}$ satisfies an LDP with rate function
    \begin{equation}
   J_\Delta(t) =  \KL(T(q, p, \beta(t)) \| \phi_\delta ), 
   \label{eq:J-delta}
\end{equation}    
where 
$
    \beta(t) = \arg_{\beta \in \mathbb{R}} \{ H(T(q, p, \beta) \| q) = t \}.
$
\label{thm:LDP-reward}
\end{theorem}

\begin{proof}[Proof sketch of Theorem~\ref{thm:LDP-reward}] %
The proof relies on a connection with the LDP of mismatched information along with a change of measure along the aligned family to $\phi_\delta^m$.
The LDP for the mismatched information is itself proven using a proof technique developed in ~\citep[Theorem 2]{salamatian2019mismatched}. More precisely, one first constructs sets in the space of distribution and argue that the empirical type of a sequence falling within these sets constrains the reward. Finally, a careful application of Sanov's Theorem~\citep[Theorem~6.2.10]{Dembo} concludes the argument.
\end{proof}

As we mentioned above, one practical usecase for the LDP of the reward, is the derivation of the scaled cumulant of function via Varadhan's lemma. %

\begin{corollary}[cumulants of the reward under aligned model]
\label{cor:varadhan_ldp}
    For all $\rho \geq 0,$\footnote{The definition is continuously extended at $\rho = 0$.} the scaled cumulants of the reward are given by
    \begin{equation}
      \lim_{m \to \infty} \frac{1}{m} \frac{1}{\rho} \log E_{Y^m \sim \phi^m_\delta} e^{\rho \, r(Y^m)} = - H_{1+\rho} (\phi_\delta \| q).
    \end{equation}
\end{corollary}
Note that the same result could have been directly derived from the characterization of the cumulants in Lemma~\ref{lem:cumulant-general} as well.
This immediately implies that the average reward of a draw from $\phi_\delta$ concentrates on $ H (\phi_\delta \| q):$
 \begin{equation}
      \lim_{m \to \infty} \frac{1}{m}   E_{Y^m \sim \phi^m_\Delta} r(Y^m) = - H (\phi_\delta \| q).
    \end{equation}
This characterizes the expected reward that is achieved by the solution to the KL-constrained RL problem.

\subsection{Best-of-$N$ Alignment}

Next, we provide theoretical guarantees on the best-of-$N$ alignment method and relate it to the optimal KL-constrained alignment problem. We first start with proving some properties of the best-of-$N$ policy. 

Under Assumptions \ref{assump:memoryless}-\ref{assump:linear-reward}, let $\pi^m_N$ denote best-of-$N$ distribution on sequences of length $m.$
Intuitively, one would think that $\pi^m_N$ still retains the nice memoryless property of $p^m$; we show below that, sadly this property is lost.

\begin{example}[The set of memoryless sources is not closed under best-of-$N$]
  \label{example:best-of-N-not-product}
  Let $m=2$ and $N=2$. Let $p$ be a uniform distribution over
  $\mathcal{Y}= \{ 0, 1, 2 \}$: $p (y_1 = 0) = 0.2$, $p (y_1 = 1) = 0.3$, $p(y_2 = 2) = 0.5$. Let the reward function be $r (y = 0) =
  \log_e 6$, $r (y = 1) = 0$ and $r (y = 2) = \log_e 2$. Additionally,
  suppose best-of-$N$ will pick uniformly for tie-breaking. We can compute
  directly $\pi^m_N$ and show that it is no longer i.i.d.:
  $\pi^m_N (y_1 = 0) = \pi^m_N(y_2 = 0) = 209/625$, while, 
  \[
  \pi^m_N (y_1 = 0, y_2 = 0) = 49/625 \neq \pi^m_N (y_1 = 0) \pi^m_N(y_2 = 0).
  \]
  For completeness, we list $\pi^m_N$ for all outcomes in Table~\ref{tab:product}.
  \begin{table}[h]
  \centering
  \caption{Full description of $\pi_N^m$ for Example~\ref{example:best-of-N-not-product}}
  \begin{tabular}{cccc}
    $\pi_N^m (y_1, y_2)$ & $y_1 = 0$ & $y_1 = 1$ & $y_1 = 2$\vspace{0.05in}\\
    \hline
    $y_2 = 0$ & 49/625 & 21/250 & 43/250\\
    $y_2 = 1$ & 21/250 & 81/10000 & 9/125\\
    $y_2 = 2$ & 43/250 & 9/125 & 103/400\\
  \end{tabular}
  \label{tab:product}
\end{table}
\end{example}

Even though $\pi^m_N$ is not a memoryless distribution, we show that it still is symmetric and is an exchangeable distribution.
\begin{lemma}[best-of-$N$ alignment model is exchangeable]
  If Assumptions~\ref{assump:memoryless}-\ref{assump:linear-reward} hold and the ties in best-of-$N$ algorithm are broken uniformly at random, then   $\pi^m_N$ is an exchangeable distribution, i.e., \[\pi^m_N(X_i=x_i, \ldots, X_j=x_j) = \pi^m_N(X_i=x_j, \ldots, X_j=x_i),\] for any $i \neq j$.
\label{lem:exchangeable}
\end{lemma}
\begin{proof}[Proof of Lemma~\ref{lem:exchangeable}]
    Without loss of generality, we prove for $i = 1$ and $j = 2$. We need to show
    \[ P [\pi_N^m = (x_1, x_2, \ldots, x_m)] = P [\pi_N^m = (x_2, x_1, \ldots,
       x_m)] . \]
    Let $\bp(\cdot | \bx) = p^m$. Let $p_c^m$ be a coupling of $p^m$ such that it will swap first item and
    second item of $p^m$. Since $p^m$ is memoryless and product of identical distributions, the probability mass function of $p_c^m$ and $p^m$ are
    the same. By running best-of-$N$ algorithm over $p_c^m$ and $p^m$, we can see that their
    first and second item are swapped while having the same probability mass function for their     corresponding best-of-$N$.
\end{proof}
\ssedit{In view of the previous example, it may seem hopeless to relate best-of-$N$ aligned model and the optimal aligned model, given that the latter remains a memoryless source whereas best-of-$N$ does not even satisfy that criterion.}
However, we show that for certain values of $m$ and $N$, both the optimal KL-constrained RL alignment and best-of-$N$ alignment would roughly be doing the same thing, which we establish using Lemma~\ref{lem:almost-optimal-aligned}.

To this end, we first state a KL divergence  upper bound for the best-of-$N$ strategy.
\begin{lemma}[KL divergence upper bound for best-of-$N$] For any $N,$ $\bx,$ $m,$ and $p,$
    \begin{equation}
           \KL(\pi^m_N(\cdot \| \bx) \| p^m(\cdot \| \bx))
   \leq \log N.
    \end{equation}
\end{lemma}
\begin{proof}
    This is a direct corollary of \citep[Theorem 1]{beirami2024theoretical}.
\end{proof}

Given this result, we set $N = \exp(m \delta)$ such that the optimal KL-constrained aligned model and best-of-$N$ both satisfy the same upper bound on their KL divergence, i.e., $\Delta = m\delta$. To show that they are close to each other in KL divergence using Lemma~\ref{lem:almost-optimal-aligned}, we need to establish that they roughly achieve the same expected reward, which then implies the KL divergence between the two distributions $\pi^m_N$ and $\phi_\delta^m$ tends to zero. We show that not only the two methods have the same reward asymptotically, but their output types are also similar.
\begin{restatable}{lemma}{lemmaTypeConvergence}
    \label{lemma:type_convergence} Let Assumptions~\ref{assump:memoryless}-\ref{assump:linear-reward} hold.
    Let $Y^m \sim \pi_N^m$ be a sequence generated from the best-of-$N$ method.  Let $\delta \in [0, \log \frac{1}{\zeta}]$, and $N = \exp(m\delta)$. Denote by $t(Y^m)$ the type of $Y^m$ (Definition~\ref{def:type}). Then, for any $\epsilon > 0$, we have:
    \begin{equation}
        \lim_{m \to \infty}\mathbb{P}(|t(Y^m) - \phi_\delta| \leq \epsilon) \to 1.
    \end{equation}
\end{restatable}
The above result implies the following result on the reward of the best-of-$N$ model.
\begin{lemma}[best-of-$N$ achieves the same reward as optimal aligned model]
\label{cor:same_reward}
Let Assumptions~\ref{assump:memoryless}-\ref{assump:linear-reward} hold. Let $\delta \in [0, \log \frac{1}{\zeta}]$  and $N = \exp(m \delta)$.  
Then 
\[
\lim_{m \rightarrow \infty} \frac{1}{m} H(\pi^m_N(\cdot | \bx) \| q^m(\cdot | \bx)) =\lim_{m \to \infty} \frac{1}{m} H(\phi^m_\delta (\cdot | \bx) \| q^m(\cdot | \bx))
 \]
\end{lemma}
\begin{proof}
Let $Y^m \sim \pi_N^m$.
We first relate a sample's type to its reward:
\[\mathbb{E}_{x \sim \Type(Y^m)} [\br (x)] = \sum_{x \in \bm{\mathcal{Y}}}
  \frac{N_x}{m} \br (x) = \frac{1}{m} \, \br (Y^m), \] where $N_x$ is the count of
  element $x$ over the stream of $m$ tokens of $Y_i^m$.
  By \Cref{lemma:type_convergence}, the type of best-of-$N$ converges to $\phi_\delta$, and hence, \begin{align*}
        \lim_{m\rightarrow \infty}\frac{1}{m} \left( H(\pi^m_N(\cdot | \bx) \| q^m(\cdot | \bx))  -  H(\phi^m_\delta (\cdot | \bx) \| q^m(\cdot | \bx)) \right)
     &   = -
\lim_{m\rightarrow \infty}\frac{1}{m} \, \br (Y^m) - \mathbb{E}_{x \sim \phi_\Delta} [\br (x)] 
         = 0,
  \end{align*}
which completes the proof.
\end{proof}

We now state our main theorem, which says that the best of $N=\exp(m \delta)$ is asymptotically the same to that of the optimal solution of KL-constrained RL alignment.
\begin{theorem}
\label{thm:kl_divergence}
Under Assumptions \ref{assump:memoryless}-\ref{assump:linear-reward},
let $\phi_\delta^m$ be the optimal solution to \Cref{definition:RL_alignment}, and $\pi^m_{N}$ be the distribution of the best-of-$N$. For any $\delta \in [0, \log \frac{1}{\zeta}]$, if $N=\exp(m \delta)$, then we have that for all $\bx$,
    \begin{equation}
    \lim_{m  \to \infty}  \frac{1}{m}  \KL(\pi^m_N(\cdot | \bx) \| \phi^m_\delta(\cdot | \bx)) = 0.
    \end{equation}
\end{theorem}
\begin{proof}
    Combining Lemma~\ref{cor:same_reward} with Lemma~\ref{lem:almost-optimal-aligned} yields Theorem~\ref{thm:kl_divergence}.
\end{proof}
Theorem~\ref{thm:kl_divergence} shows that the best-of-$N$ distribution is asymptotically close to the KL-constrained RL solution.
We remark that even though best-of-$N$ is not necessarily a product distribution as shown in \Cref{example:best-of-N-not-product}, it still is close to the optimal solution of the KL-constrained RL, which is a product distribution in our setting.  
We believe that Theorem~\ref{thm:kl_divergence} sheds some light on the remarkable performance of the best-of-$N$ method when evaluated on the tradeoffs between expected reward and the KL divergence between the aligned model and the reference model (reward-KL tradeoffs).

We also empirically validate that the type of best-of-$N$ is close to the solution for the KL-constrained RL problem for $N = \exp(m\delta)$. To this end, we revisit the ternary example in Figure~\ref{fig:simplex}. As can be seen, even for modest values of $m$, and $N$, we observe that $\mathbb{E}[t(Y^m)]$ where $Y^m \sim \pi^m_N$ is remarkably close to $\phi_\delta$.

\ssedit{We have thus shown that the typical behavior of the best-of-$N$ method is mirrored by the solution to the KL-constrained RL. We conclude this section with a {\bf conjecture}: we propose that the atypical behavior of the best-of-$N$ could also be described by a Large Deviation Principle (LDP) with the identical rate function as that of the KL-constrained RL, as detailed in Theorem~\ref{thm:LDP-reward}.}

\forfuture{In other words,
\begin{itemize}
    \item The implication of this result is that best-of-$N$ approximately solves the KL-constrained RL problem under assumptions of this paper.
    \item We can empirically argue that $m$ doesn't have to be too large for this to be true. Also, empirically we can relax the memoryless source assumption.
    \item For now, I thought of $p^m$ to be an i.i.d. distribution, and similarly $q^m$ to be i.i.d. (which is questionable) but I think this should be easily extendable to finite-state machine sources.
    \item I think it would be interesting to this community to characterize the rate at which~\eqref{eq:rate} vanishes. At the very least we can define the problem.
\end{itemize}}

\section{Concluding Remarks}
In this paper, we considered the language model alignment problem, and studied two popular alignment techniques: {\em KL-constrained reinforcement learning} and {\em best-of-$N.$} We established several properties of the solutions of both techniques and established an asymptotic equivalence between the two problems, providing theoretical justification for the remarkable performance of best-of-$N$ in practice. While our derivations were obtained under strong assumptions about the string-generating sources, we believe that these results may be extended beyond these classes of sources. In particular, we performed some extra experiments to compare the KL divergence of best-of-$N$ solution, $\bpi_N$ to that of the optimal KL-regularized RL where the KL divergence budget for both models is $\Delta$, i.e., $\bphi_\Delta$. We found that for random $\bp$ and $\bq$ over alphabet of size $K>1000,$ the KL divergence $\KL(\bpi_N \| \bphi_\Delta)$ was always bounded by $0.01$ for all $N \in [1, 1000]$ which is striking. Even for $K<10,$ we empirically observed that $\KL(\bpi_N \| \bphi_\Delta) < 0.5.$  We hope future research can extend the theoretical investigation to better understand this phenomenon.

\section*{Acknowledgements}
The authors are grateful to Alekh Agarwal, Ananth Balashankar, Jonathan Berant, Michael Collins, Jacob Eisenstein, Adam Fisch, Jong Lee, and Mahyar Jafarinodeh for helpful discussions on the fundamentals of language model alignment.

\bibliography{alignment}
\bibliographystyle{template_conference}

\clearpage

\appendix

\section{Proofs of Properties of KL-constrained RL}
\label{app:solution_to_KL_RL}
\begin{proof}[Proof of \Cref{lemma:solution_to_KL_RL}]
    The Lagrangian is written as
    \begin{equation}
     \mathcal{L}(\bphi) =  H (\bphi \|\bq) + \lambda \left( \KL (\bphi \|\bp) - \Delta \right) .
    \end{equation}
    Fix $\lambda > 0$ we can solve for its minimum with
    respect to $\bphi$ following steps,
    \begin{align}
       \min_{\bphi} \mathcal{L}(\bphi)  & = \min_{\bphi} H (\bphi \|\bq) + \lambda \left( \KL (\bphi \|\bp) - \Delta \right) \\
       & = \min_{\bphi}  \sum_{\by \in {\mathcal{Y}}} \bphi (\by)  \left( \log \frac{1}{\bq (\by)} + \lambda \log
      \frac{\bphi (\by)}{\bp (\by)} - \lambda \Delta 
      \right)\\
      & = \min_{\bphi} \lambda \sum_y \bphi (\by)  \left(
      \log \frac{\bphi (\by)}{\bp (\by)\bq (\by)^{1 / \lambda}} - \Delta \right)\\
      & =  \min_{\bphi} \lambda \sum_{\by} \bphi (\by) \left( \log \frac{\bphi (\by)}{T (\bq, \bp, 1 /
      \lambda) (\by)} \right) + C(\lambda, \Delta),\\
      & =  \min_{\bphi} \lambda \KL( \bphi (\by) \| T (\bq, \bp, 1 /
      \lambda) (\by)) + C(\lambda, \Delta), 
    \end{align}
which is uniquely minimized by $\bphi (\by) = T (\bq, \bp, 1 / \lambda) (\by)$.
\end{proof}

\section{Proofs of Large Deviation Principle (LDP)}
Now to prove LDP, we first show some results on the KL-constrained RL. 
\begin{theorem}[LDP for mismatched information]
    Let $Y^m \sim p^m.$ 
    The sequence $\{ -\frac{1}{m} \log q^m(Y^m)\}_{m \in \mathbb{N}}$ satisfies an LDP with rate function
    \begin{equation}
   J_0(t) =  \KL(T(q, p, \beta(t)) \| p ), 
   \label{eq:J-delta-info}
\end{equation}    
where 
\begin{equation}
    \beta(t) = \arg_{\beta \in \mathbb{R}} \{ H(T(q, p, \beta) \| q) = t \}.
\end{equation}
\label{thm:mismatched-LDP}
\end{theorem}
The proof of Theorem~\ref{thm:mismatched-LDP} relies on a correspondence between mismatched information, and some sets of distributions, which we will define shortly. This correspondence is implicitly used in \citep[proof of Theorem 5]{beirami2018characterization} and made explicit in \citep[Eq. (25)-(27)]{salamatian2019mismatched}. For $\epsilon \geq 0$ and $\alpha \in \mathbb{R}$, let 
\begin{align}
    \mathcal{D}(q,\alpha,\epsilon) & \triangleq \left\{ \varphi \in \Delta_{\mathcal{X}}: H(\varphi \| q) - H(T(q,\alpha)\| q) \leq  \epsilon \right\} \\
    \mathcal{E}(q,\alpha,\epsilon) & \triangleq \left\{ \varphi \in \Delta_{\mathcal{X}}: H(\varphi \| q) - H(T(q,\alpha)\| q) \geq  - \epsilon \right\} \\
    \mathcal{B}(q,\alpha,\epsilon) &\triangleq \left\{\varphi \in \Delta_{\mathcal{X}}:H(T(q,\alpha)\|q) - H(\varphi \| q) \in [0,\epsilon]\right\}.
\end{align}
The sets above are extensions of tilted weakly typical sets of order $\alpha$~\citep[Definition 18]{beirami2018characterization}, and capture the set of types which are respectively, more likely, less likely, and as likely according to $q$ than $T(q,\alpha)$, where we define
\begin{equation}
    T(q, \alpha) := T(q, u, \alpha), 
\end{equation}
where $u = (1/|\mathcal{Y}|, \ldots, 1/|\mathcal{Y}|)$ is the uniform distribution.
For these sets, we then have the following lemma.

\begin{lemma}
For any $\alpha > 0$, the following inclusion relations hold, for sufficiently large $n$,
    \begin{align}
    & \left| - \frac{1}{m} \log q(y^m) - H(T(q,\alpha)\| q) \right| \leq \epsilon  \Rightarrow  \Type(y^m) \in \mathcal{D}(q,\alpha,2\epsilon), \label{eq:D_set}\\
    & \left| - \frac{1}{m} \log q(y^m) - H(T(q,\alpha)\| q) \right| \leq \epsilon  \Rightarrow  \Type(y^m) \in \mathcal{E}(q,\alpha,2\epsilon), \label{eq:E_set}\\
    &\left| - \frac{1}{m} \log q(y^m) - H(T(q,\alpha)\| q) \right| \leq \epsilon \Leftarrow \Type(y^m) \in \mathcal{B}(q,\alpha,\epsilon).\label{eq:B_set}
\end{align}
\end{lemma}
This was proved implicitly in the proofs of Theorems 3 and 5 in~\citep{beirami2018characterization}.
We are now equipped to provide the proof of the main theorem.

\begin{proof}[Proof of Theorem~\ref{thm:mismatched-LDP}]

Note that as $-\frac{1}{m}\log q(Y^m)$ 
takes values in a compact subset $[0,\log \frac{1}{\zeta}]$ of $\mathbb{R}$ (see Assumptions \ref{assump:memoryless}-\ref{assump:linear-reward}), it is sufficient to prove that the limit below exists and evaluates to the rate function~(see \citep[Section V]{beirami2018characterization} for a formal discussion), i.e., 
	\begin{align}
\lim_{\epsilon \downarrow 0} \lim_{m \to \infty} \frac{1}{m}\log  \mathbb{P}_p^m\left(  \left|- \frac{1}{m} \log q(Y^m) - t \right| < \epsilon \right)=  -J(t).
	\end{align}
We proceed with the proof in three separate cases. 

{\it Case (a)}:  We let $t \in (H(q),\log \frac{1}{\zeta})$, which implies $\alpha(t) \in (0,1)$ by monotonicity of $H(T(q,\alpha)\| q)$ for $\alpha \in \mathbb{R}$. Note that
\eqref{eq:D_set} and \eqref{eq:B_set} respectively imply
\begin{align}
   & \lim_{\epsilon \downarrow 0} \limsup_{m \to \infty} \frac{1}{m} \log \mathbb{P}_{p^m}\left( \left| - \frac{1}{m} \log q(Y^m) - H(T(q,\alpha(t))\|q) \right| \leq \epsilon \right) \notag \\
   & \hspace{.5em} \leq \lim_{\epsilon \downarrow 0} \limsup_{m \to \infty} \frac{1}{m} \log \mathbb{P}_{p^m}(\Type(Y^m) \in \mathcal{D}(q,\alpha(t),2\epsilon)), \label{eq:ldp_upb}\\
   & \lim_{\epsilon \downarrow 0} \liminf_{m \to \infty} \frac{1}{m} \log \mathbb{P}_{p^m}\left( \left| - \frac{1}{m} \log q(Y^m) - H(T(q,\alpha(t))\|q) \right| \leq \epsilon \right) \notag \\
   & \hspace{.5em} \geq \lim_{\epsilon \downarrow 0} \liminf_{m \to \infty} \frac{1}{m} \log \mathbb{P}_{p^m}(\Type(Y^m) \in  \mathcal{B}(q,\alpha(t),\epsilon)). \label{eq:ldp_lbd}
\end{align}
Thus, it suffices to show that the RHS of \eqref{eq:ldp_upb} and \eqref{eq:ldp_lbd} both evaluate to $-D(T(q, p, \alpha(t)) \| p)$. This is done via Sanov's Theorem. Recall that Sanov's Theorem~\citep[Theorem~6.2.10]{Dembo} states that, for a set of distributions $\mathcal{C}$, 
    	\begin{align}
	    - \inf_{\gamma \in \mathrm{int} \mathcal{C}} D(\gamma \| p) &\leq \liminf_{m \to \infty} \frac{1}{m} \log \mathbb{P}(\Type(y^m) \in  \mathcal{C}) \notag \\
	    &\leq \limsup_{m \to \infty} \frac{1}{m} \log \mathbb{P}(\Type(y^m) \in \mathcal{C}) \notag \\
	    & \leq - \inf_{\gamma \in \mathrm{cl} \mathcal{C}} D(\gamma \| p).
	\end{align}
	To obtain the upper bound, we apply this result to the set $\mathcal{D}(q,\alpha(t),2\epsilon/\alpha(t))$. Observing that this holds for any $\epsilon$, and then letting $\epsilon \downarrow 0$, we get that the RHS of \eqref{eq:ldp_upb} is upper bounded 
	\begin{align}
	    - \lim_{\epsilon \downarrow 0} \inf_{\gamma \in \mathrm{cl} \mathcal{D}(q,\alpha(t),2\epsilon/\alpha(t))} D(\gamma \| p). \label{eq:lim_eps_to_0}
	\end{align}
	We now make use of a basic topological fact. Observe that $D(\gamma \| p)$ is strictly convex in $\gamma$ for a fixed $p$, and thus there is a unique minimizer $\gamma(t,\epsilon)$. Noting that the minimizer $\gamma(t, \epsilon)$ is in the set $\mathrm{cl}\mathcal{D}(q,\alpha(t),2\epsilon)$, by continuity of $D(\gamma \| p)$ and compactness of the set. Thus, the collection of minimizers $\gamma(t,\epsilon)$ is a collection of points such that $\gamma(t,\epsilon) \in \mathrm{cl}\mathcal{D}(q,\alpha(t),2\epsilon)$. It follows from compactness that the limit point $\lim_{\epsilon \downarrow 0} \gamma(\epsilon,t) \in \bigcap_{\epsilon > 0} \mathrm{cl}\mathcal{D}(q,\alpha(t),2\epsilon) = \mathrm{cl} \mathcal{D}(q,\alpha(t),0)$, where we have used that $\alpha(t) > 0$. Therefore, we have the bound
	\begin{align}
	    & \inf_{\gamma \in \Delta_{\mathcal{X}}} \hspace{3em} D(\gamma \| p) \notag \\
	    & \text{subject to } \hspace{1.5em} H(\gamma\|q) \leq H(T(q,\alpha(t))\| q)
	\end{align}
	Note that this optimization problem is convex, and thus can be solved analytically by writing the KKT conditions~\citep{boyd2004convex}, which give a solution $\gamma_{q,p}(t) \in \mathcal{T}_{q,p}$, and optimal value $D(\gamma_{q,p}(t)\| p)$.
	
	Analogously, the RHS of \eqref{eq:B_set} can be shown to be lower bounded by $- \inf D(\gamma \| p)$, where $\gamma \in \mathcal{B}(q,\alpha(t),0)$, by noting $\mathcal{B}(q,\alpha,0) \subset \mathrm{int} \mathcal{B}(q,\alpha,\epsilon)$, for any $\epsilon > 0$. Again, this optimization can be solved analytically, and gives the desired output. Putting these results together, we get that
	\begin{align}
\lim_{\epsilon \downarrow 0} \lim_{m \to \infty} \frac{1}{m}\log  &\mathbb{P}_p^m\left(  \left|-\frac{1}{m} \log q(Y^m) - H(T(q,\alpha(t))\| q)\right| < \epsilon \right) \notag \\ 
&= -D(\gamma_{q,p}(t) \| p).
	\end{align}
	
	{\it Case (b):} We now let $t \in (0,H(q))$, which implies $\alpha(t) \in (1,\infty)$. The proof in this case follows from the same step as in Case (a), by replacing the set $\mathcal{D}(q,\alpha(t),\epsilon)$ with the set $\mathcal{E}(q,\alpha(t),\epsilon)$.
	
	{\it Case (c):} Finally, let $t = H(q)$, or equivalently, $\alpha(t) = 1$. In this case, note that $p \in \mathcal{B}(q,1,\epsilon)$, and thus, by the law of large numbers and \eqref{eq:ldp_lbd}, we have that
\begin{align}
    \lim_{\epsilon \downarrow 0} \lim_{m \to \infty} \frac{1}{m}\log  \mathbb{P}_p^m\left(  \left|-\frac{1}{m} \log q(Y^m) - t \right| < \epsilon \right) \geq 0,
\end{align}	
	which implies that $J(t) = 0$ in this case.
	\end{proof}

\begin{proof}[Proof of Theorem~\ref{thm:LDP-reward}]
    The proof is completed by invoking Theorem~\ref{thm:mismatched-LDP} with $q$ replaced with $\phi_\Delta$, noticing that $\phi_\Delta = T(q, p, \Delta),$ and invoking Lemma~\ref{lem:tilt-tilt} to express $T(q, \phi_\Delta, \beta)$ in terms of $p$.
\end{proof}

\begin{proof}[Proof of Corollary~\ref{cor:varadhan_ldp}]
    We use Varadhan’s Lemma \citep[Theorem 4.3.1]{Dembo}, which states that if a sequence of random variables $A_m$ satisfies a LDP with rate function $J(t)$, then we have:
    \begin{equation}
        \lim_{m \to \infty} \frac{1}{m} \log \mathbb{E}[\exp(m F(A_m))] = \sup_t \{F(t) - J(t)\},
    \end{equation}
    for any continuous and bounded function $F$. Applying this result with $F(t) = \gamma \cdot t$, we obtain:
    \begin{align}
        \lim_{m \to \infty} &\frac{1}{m} \log E_{Y^m \sim \phi^m_\Delta} e^{\gamma r(Y^m)}   \label{eq:varadhan_step1}\\ = &\sup_t \; \{\gamma H(T(q, p, \beta(t)) \| q) - D(T(q, p, \beta(t)) \| \phi_\Delta) \}\nonumber \\
        = &\sup_\beta \; \{\gamma H(T(q, p, \beta \| q) - D(T(q, p, \beta \| \phi_\Delta) \}\label{eq:varadhan_step2},
    \end{align}
    where we used the definition of $\beta(t)$ in Theorem~\ref{thm:LDP-reward} in \eqref{eq:varadhan_step1}, and changed the optimization variable to $\beta$ in \eqref{eq:varadhan_step2}. Expanding the RHS, and collecting terms, we get:
    \begin{align}
         & \; - \sum_y  T(q, p,\beta)(y)   \log \frac{T(q, p, \beta)(y)}{\phi_\Delta(y) \cdot q(y)^\gamma} \nonumber \\
        = & - \sum_y T(q, p, \beta)(y) \log \frac{T(q, p, \beta)(y)}{T(q, \phi_\Delta, \gamma)(y)} \nonumber \\ & - \log \sum_y \phi_\Delta(y)q(y)^\gamma \\
        = & - \KL(T(q,p, \beta) \| T(q, \phi_\Delta, \gamma)) - \log \sum_y \phi_\Delta(y)q(y)^\gamma. \nonumber
    \end{align}
    This is maximized when $T(q, p, \beta) = T(q, \phi_\Delta, \gamma)$, which is possible since $\phi_\Delta$ is on the aligned family between $p$ and $q$. The proof follows from identifying the remaining term as $(\gamma+1) H_{\gamma+1}(\phi_\Delta \| q)$. %
\end{proof}

\section{Proofs of Properties of Best-of-$N$}
Let $K = |\mathcal{Y}|$. For a sequence $y^m,$ let $t({y^m})$ denote its type.
The proof of this notion of $m$-grained types, which is defined below. 
\begin{definition}
    A discrete distribution $p$ on $\mathcal{Y}$ is called $m$-grained if for all elements, its probability are multiples of $\frac{1}{m}$, i.e., $\forall i \in \bm{\mathcal{Y}}, \exists c \in \mathbb{N}, p(i) = c \cdot \frac{1}{m}$.
\end{definition}

\lemmaTypeConvergence*
\begin{proof}[Proof of \Cref{lemma:type_convergence}]
    Denote by $\mathcal{E} = \{ \lambda : |\lambda - \phi_\delta| > \epsilon \}$, let $\mathcal{T}$ be the set of $m$-grained distributions, we note that:
    \begin{align}
        \mathbb{P}(|t(Y^m) - \phi_\delta| > \epsilon) & = \mathbb{P}(t(Y^m) \in \mathcal{E}) \\
        & \leq \sum_{\lambda \in \mathcal{E} \cap \mathcal{T}} \mathbb{P}(t(Y^m) = \lambda).
    \end{align}
   Therefore, if we prove that for any valid $\lambda \in \mathcal{E}\cap\mathcal{T}$, the probability of $Y^m$ having type $\lambda$ is exponentially small (in $m$), our proof is complete by observing that there are only a polynomial number of valid types.

    Recall that the best-of-$N$ is selected such that $Y^m$ is the sequence among $ X_1^m,\ldots, X_N^m \overset{i.i.d.}{\sim} \bp$ which has the highest reward, i.e. $r(Y^m) = \max \{ r(X_1^m), \ldots, r(X)\}$
   For any valid type $\lambda \in \mathcal{T}$, we thus have:
   \begin{align}
       \mathbb{P}(t(Y^m) = \lambda) & = N \mathbb{P}(t(X^m_N) = \lambda) \prod_{i=1}^{N-1}\mathbb{P}(r(X^m_i) \leq r(\lambda)). \label{eq:best_of_N_exact_prob}
   \end{align}
   Next, note that $r(X^m) \leq r(\lambda)$ can be written equivalently in terms of the type of $X^n$ as $H(t(X^m) \| q) \leq H(\lambda \| q)$. Letting the set $\mathcal{A}(\lambda) = \{ t: H(\lambda \| q) > H(t \| q) \}$,  the last term in \eqref{eq:best_of_N_exact_prob} is equivalently written as:
   \begin{align}
       \mathbb{P}(r(X^m) \leq r(\lambda)) & = \mathbb{P}(t(X^m) \in \mathcal{A}(\lambda)) \\
       & \leq 1 -  \exp \{ - m D(T(p, q, \alpha(\lambda)) \| p)\},
   \end{align}
where $\alpha(\lambda)$ is the unique tilt so that $H(T(p, q, \alpha(\lambda)) \| q ) = H(\lambda \| q)$ (by Sanov's Theorem). Therefore, the product in \eqref{eq:best_of_N_exact_prob} can be bounded by:
\begin{align}
    \prod_{i=1}^{N-1}\mathbb{P}(r(X^m_i) \leq r(\lambda)) &\leq \left(1 - \exp\left\{-m \left( D(T(p, q, \alpha(\lambda)) \| p) \right)\right\}\right)^N \\
    & \leq \exp{-m \left[ D(T(p, q, \alpha(\lambda)) \| p) - \delta\right]_+}
\end{align}
On the other hand, if $N$ is sufficiently large, precisely if $\delta > D(T(p, q, \alpha(\lambda)) \| \bp)$, another bound is tighter, namely:
\begin{align}
    \prod_{i=1}^{N-1}\mathbb{P}(r(X^m_i) \leq r(\lambda)) &\leq \left(1 - \exp\left\{-m \left( D(T(p, q, \alpha(\lambda)) \| p) \right)\right\}\right)^N \\
    & = \left( \left(1 - \frac{1}{x}\right)^x \right)^{N/x}  \\
    & \leq \exp{-(\exp{m (\delta - D(T(p, q, \alpha(\lambda) \| p)}))}
\end{align}
where we used $(1-1/x)^x \leq 1/e$, and $x = \exp{m D(T(p, q, \alpha(\lambda)) \| p)}$.
Therefore, we obtain the general bound:
\begin{align}
    \prod_{i=1}^{N-1}\mathbb{P}(r(X^m_i) \leq r(\lambda)) \leq \exp - m F(m, \alpha(\lambda), \delta, p, q),
\end{align}
where
\begin{align}
    F(m, \alpha(\lambda), \delta, p, q) = \begin{cases}
     D(T(p, q, \alpha(\lambda)) \| p) - \delta & \text{if } D(T(p, q, \alpha(\lambda)) \| p) > \delta \\
    \frac{1}{m}\exp {m (\delta -  D(T(p, q, \alpha(\lambda)) \| p))} & \text{if } D(T(p, q, \alpha(\lambda)) \| p) < \delta \\
    0 & \text{if } D(T(p, q, \alpha(\lambda)) \| p) = \delta 
    \end{cases}\label{eq:F_function}
\end{align}
Then, putting everything together in \eqref{eq:best_of_N_exact_prob}, we get:
   \begin{align}
       \mathbb{P}(t(Y^m) = \lambda) & \leq  \exp \left\{-m\left[D(\lambda \| p) + F(m, \alpha(\lambda), \delta, p, q)  - \delta \right]\right\} \\
       & =  \exp \left\{-m\left[D(\lambda \| T(p, q, \alpha(\lambda)) + D(T(p, q, \alpha(\lambda)) \| p) + F(m, \alpha(\lambda), \delta, p, q)  - \delta \right]\right\}  \label{eq:I-pythagorean theorem}
   \end{align}
   where \eqref{eq:I-pythagorean theorem} follows from the I-Pythagorean Theorem.
   For \eqref{eq:I-pythagorean theorem} to decay to zero as $m \to \infty$, we must have the exponent be strictly positive. We check this in three steps:

   \begin{itemize}
       \item  Let $D(T(p, q, \alpha(\lambda)) \| p) > \delta$, and identifying \eqref{eq:F_function}, the exponent in \eqref{eq:I-pythagorean theorem} is:
       \begin{align}
           D(\lambda \| T(p, q, \alpha(\lambda)) + 2(D(T(p, q, \alpha(\lambda)) \| p) - \delta) > 0
       \end{align}
       and therefore $\mathbb{P}(t(Y^m) = \lambda)$ goes to zero exponentially fast.
       \item Let $D(T(p, q, \alpha(\lambda)) \| p) < \delta$, then, the exponent in \eqref{eq:I-pythagorean theorem} is:
       \begin{align}
           D(\lambda \| T(p, q, \alpha(\lambda)) + D(T(p, q, \alpha(\lambda)) \| p)- \delta +  \frac{1}{m}\exp {m (\delta -  D(T(p, q, \alpha(\lambda)) \| p))}  > 0,
       \end{align}
       for $m$ large enough.
    \item Finally, let $D(T(p, q, \alpha(\lambda)) \| p) = \delta$, then the exponent becomes:
    \begin{align}
        D(\lambda \| T(p, q, \alpha(\lambda))),
    \end{align}
    which is positive if and only if $\lambda = T(p, q, \alpha(\lambda))$.
   \end{itemize}
   Putting these cases together, we obtain that $\mathbb{P}(t(Y^m))$ goes to 0 exponentially fast unless $\lambda = T(p,q, \alpha(\delta))$.  
\end{proof}

\begin{lemma}
 Let $\alpha, \beta \in \mathbb{R},$ then:   
\begin{equation}
    T(q, T(q, p, \alpha), \beta)) = T(q, p, \alpha+\beta).    
\end{equation}
\label{lem:tilt-tilt}
\end{lemma}
\begin{proof}
 The proof is completed by noticing
 \begin{align}
    T(q, T(q, p, \alpha), \beta))  & \propto  T(q, p, \alpha) q^\beta \nonumber\\
    & \propto p q^{\alpha \beta} = T(q, p, \alpha + \beta).
 \end{align}      
\end{proof}

\end{document}